\documentclass[twoside,american]{elsarticle}
\usepackage[T1]{fontenc}
\usepackage[utf8]{inputenc}
\usepackage{amsmath}
\usepackage{amsthm}
\usepackage{amssymb}
\makeatletter
\journal{}

\usepackage{babel}
\providecommand{\definitionname}{Definition}
\providecommand{\theoremname}{Theorem}

\makeatletter
\def\ps@pprintTitle{%
 \let\@oddhead\@empty
 \let\@evenhead\@empty
 \let\@oddfoot\@empty
 \let\@evenfoot\@empty}
\makeatother

\makeatother

\theoremstyle{plain}
\newtheorem{thm}{\protect\theoremname}
\theoremstyle{definition}
\newtheorem{defn}[thm]{\protect\definitionname}
\theoremstyle{remark}
\newtheorem{rem}[thm]{\protect\remarkname}
\theoremstyle{plain}
\newtheorem{cor}[thm]{\protect\corollaryname}
\theoremstyle{definition}
\newtheorem{example}[thm]{\protect\examplename}
\usepackage{babel}
\providecommand{\corollaryname}{Corollary}
\providecommand{\definitionname}{Definition}
\providecommand{\examplename}{Example}
\providecommand{\remarkname}{Remark}
\providecommand{\theoremname}{Theorem}

\begin{document}
\begin{frontmatter}
\title{Sufficient conditions for a Heuristic Rating Estimation Method application}
\author[wms]{Jacek~Szybowski\corref{cor1}}
\ead{szybowsk@agh.edu.pl}
\author[weaiib]{Konrad~Kułakowski}
\ead{kkulak@agh.edu.pl}
\author[suo]{Jiri Mazurek}
\ead{mazurek@opf.slu.cz}
\cortext[cor1]{Corresponding author}
\address[wms]{AGH University of Science and Technology, Faculty of Applied Mathematics,
Krakow, Poland}
\address[weaiib]{AGH University of Science and Technology, Faculty of Electrical Engineering,
Automatics, Computer Science and Biomedical Engineering, Krakow, Poland}
\address[suo]{Silesian University in Opava, Department of Informatics and Mathematics,
School of Business Administration in Karvina, Czech Republic}
\begin{abstract}
A series of papers has introduced the Heuristic Rating Estimation
method, which evaluates a set of alternatives based on pairwise comparisons
and the weights of reference alternatives. We formulate the conditions
under which the HRE method can be applied correctly. The research
considers both arithmetic and geometric algorithms for complete and
incomplete pairwise comparison methods. The illustrative examples
show that the estimations of inconsistency in the arithmetic variant
are optimal. 
\end{abstract}
\end{frontmatter}

\section{Introduction}

Pairwise comparison (PC) methods constitute a well-established class
of decision-making tools used to derive priorities or rankings of
alternatives based on relative judgments. Among the most prominent
representatives of this family are methods such as the Analytic Hierarchy
Process (AHP), ELECTRE, PROMETHEE, MACBETH, TOPSIS, PAPRIKA, or the
Best--Worst Method \citep{BanaECosta1994baip,Brans2005pm,Figueira2005em,Hansen2008anmf,Rezaei2015bwmc,Saaty1977asmfSIMPL,Saaty1980tahp,Yoon1987arad}.

According to the classic work \citep{Miller1956tmns}, humans are
able to process only $7\pm2$ "pieces" of information at a time.
The fundamental advantage of pairwise comparisons lies in their ability
to reduce this cognitive burden, as decision-makers are required to
compare only two alternatives at a time rather than evaluate all options
simultaneously.

The origins of pairwise comparisons can be traced back to the 13th
century, when Ramon Llull applied comparative judgments to select
a prelate from a group of candidates \citep{Colomer2011rlfa}. The
method was later formalized in the early 20th century by L. L. Thurstone
in his Law of Comparative Judgment \citep{Thurstone1927tmop}, and
gained widespread popularity following the introduction of AHP by
T. L. Saaty in the late 1970s\citep{Saaty1977asmfSIMPL}. Since then,
PC methods have been successfully applied across a broad range of
disciplines, including engineering, economics, management, and social
sciences.

The Heuristic Rating Estimation (HRE) method, proposed in the mid-2010s
\citep{Kulakowski2014hrea,Kulakowski2015hreg,Kulakowski2016note},
represents a modification of classical pairwise comparison techniques.
Its distinctive feature is the explicit differentiation between two
categories of alternatives: those with known preference values (reference
alternatives) and those whose preferences are unknown. The objective
of HRE is to estimate the unknown priorities by exploiting pairwise
comparisons both among unknown alternatives and between unknown and
reference alternatives. This structure makes HRE particularly suitable
for decision problems in which partial, yet reliable, preference information
is available in advance. Typical situations where this setting occurs
include an introduction of a new product on the market while performance
of previous products is known, elections where a performance of previous
candidates is known, etc.

Two main variants of the HRE method have been developed: the arithmetic
and the geometric approaches \citep{Kulakowski2014hrea}, \citep{Kulakowski2016note}
\citep{Kulakowski2015hreg}. In the arithmetic version of the HRE,
the priority of an alternative is computed as an arithmetic mean of
its evaluations relative to other alternatives, whereas the geometric
version relies on geometric aggregation. Existing research on HRE
has primarily focused on issues related to solution existence \citep{Kulakowski2016note},
inconsistency assessment \citep{Kulakowski2020iifi}, and methodological
extensions, including multiple-criteria, incomplete or fuzzy pairwise
comparison matrices, see also \citep{Kedzior2023mchr,Kulakowski2026mbip}.

In practical applications, however, decision-makers often face incomplete
pairwise comparison data, which raises fundamental questions regarding
the solvability and uniqueness of the resulting estimation problem.
The aim of this paper is to address this gap by providing sufficient
conditions for the applicability of the HRE method in the case of
incomplete pairwise comparison matrices. Using tools from linear algebra
and spectral theory, we establish conditions under which the arithmetic
and geometric HRE formulations admit unique solutions.

While \citet{Kulakowski2026mbip} shares a common methodological background
with studies on mean-based incomplete pairwise comparison methods
with reference values, its contribution addresses a distinct and previously
underexplored problem. Existing mean-based approaches focus on extending
the HRE method by modifying aggregation rules and broadening admissible
comparison structures. In contrast, the present study adopts a theoretical
perspective and asks under what conditions the standard arithmetic
and geometric HRE methods are well-defined. By deriving sufficient
conditions for invertibility and uniqueness using spectral radius
arguments and matrix theory, this paper fills a theoretical gap by
establishing formal guarantees for HRE applicability, independent
of any particular extension or optimization-based reformulation.

The remainder of the paper is organized as follows. Section \ref{sec:Preliminaries}
introduces preliminary concepts related to pairwise comparisons and
selected results from linear algebra. Section \ref{sec:Heuristic-Rating-Estimation}
discusses the arithmetic and geometric HRE methods in the complete
case and derives conditions ensuring solution existence. Section \ref{sec:Heuristic-Rating-Estimation-1}
extends the analysis to incomplete pairwise comparison matrices and
presents sufficient conditions for solvability in both arithmetic
and geometric HRE variants. The paper concludes with a summary of
the main results. 

\section{Preliminaries }\label{sec:Preliminaries}

\subsection{Pairwise comparisons}

The pairwise comparisons (PC) method is the process designed to transform
the set of comparisons into a ranking of alternatives. Let $A=\{a_{1},\ldots,a_{n}\}$
denote a set of alternatives while $C=[c_{ij}]$ means set of comparisons
in the form of $n\times n$ matrix, where $c_{ij}\in\mathbb{R}_{+}$
for $i,j=1,\ldots,n$. Each $c_{ij}$ means the result of direct comparison
between $a_{i}$ and $a_{j}$. When the result of the given comparison
$c_{ij}$ is unknown we write that $c_{ij}=c_{ji}=?$. 
\begin{defn}
A PC matrix $C$ is said to be reciprocal if $c_{ij}=1/c_{ji}$ for
all $i,j=1,\ldots,n$ except when $c_{ij}=?$. 
\end{defn}

The purpose of PC method is to calculate the ranking. 
\begin{defn}
Let the weight function for $A$ be $w:A\rightarrow\mathbb{R}_{+}$,
such that if $w(a_{i})>w(a_{j})$ for some $0<i,j\leq n$ then $a_{i}$
is more preferred than $a_{j}.$ 
\end{defn}

The above function prioritizes the various alternatives. The more
preferred alternatives have higher weights. The function $w$ is usually
represented as a priority vector in the form: 
\begin{equation}
w=\left(w(a_{1}),w(a_{2}),\ldots,w(a_{n})\right)^{T}.\label{eq:prior-vect}
\end{equation}
The fact that $c_{ij}$ corresponds to the ratio of the preferential
strength of alternatives $a_{i}$ and $a_{j}$ implies that one may
expect that $c_{ij}=w(a_{i})/w(a_{j})$. This in turn results in a
postulate of transitivity\footnote{As $c_{ik}=w(a_{i})/w(a_{k})$ and $c_{kj}=w(a_{k})/w(a_{j})$.}
i.e. $c_{ij}=c_{ik}c_{kj}$ for every triad $i,j,k=1,\ldots,n$. Unfortunately,
because the vector $w$ is computed based on all pairwise comparisons,
in practice we may expect only that $c_{ij}\approx w(a_{i})/w(a_{j})$,
thus $c_{ij}\approx c_{ik}c_{kj}$.

Now we recall the general definition of a consistent PC matrix, which
is valid both in a complete and incomplete case. 
\begin{defn}
A reciprocal $n\times n$ PC matrix $C$ is said to be consistent
if for each $i_{1},\ldots i_{m}\in\{1,\ldots,n\}$ 
\[
c_{i_{1}i_{2}}\cdot\ldots\cdot c_{i_{m-1}i_{m}}\cdot c_{i_{m}i_{1}}=1,
\]
provided that all the factors in the above product are known. 
\end{defn}

It is easy to prove that if $c_{ij}=c_{ik}c_{kj}$ then $c_{ij}=w(a_{i})/w(a_{j})$
\citep[p. 96]{Kulakowski2020utahp} regardless of the prioritization
method providing of course that $c_{ij}$ means the ratio between
preferential strength of $a_{i}$ and $a_{j}$.

There are at least a dozen methods for computing a vector $w$ \citep{Choo2004acff,Kulakowski2020utahp}.
The two most popular are the Eigenvalue Method (EVM) and Geometric
Mean Method (GMM). The first one was originally proposed by Saaty
in his seminal paper \citep{Saaty1977asmf} is based on the concept
of a vector and the eigenvalue of the matrix $C$. So, let $C=[c_{ij}]$
be a PC matrix containing expert judgments for $n$ alternatives,
and $w_{\textit{max}}$ be a principal eigenvector of $C$ i.e. 
\[
Cw_{\textit{max}}=\lambda_{\textit{max}}w_{\textit{max}},
\]
where $\lambda_{\textit{max}}$ is a principal eigenvalue (spectral
radius) of $C$. Thus, the priority vector $w_{\textit{ev}}$ 
is a rescaled version of $w_{\textit{max}}$ i.e. 
\[
w_{\textit{ev}}=\frac{1}{\sum^{n}_{i=1}w_{\textit{max}}(a_{i})}w_{\textit{max}}.
\]
The second method, although it is based on similar premises \citep{Kulakowski2016srot},
is easier to calculate. In GMM the priority of the $i$-th alternative
is the appropriately rescaled geometric mean of all its direct comparisons.
I.e. 
\[
w_{\textit{gm}}(a_{i})=\alpha\left(\prod^{n}_{j=1}c_{ij}\right)^{1/n},
\]

where 
\[
\alpha=\left(\sum^{n}_{i=1}\left(\prod^{n}_{j=1}c_{ij}\right)^{1/n}\right)^{-1}.
\]
Both methods EVM and GMM have their incomplete counterparts. Extending
the EVM to incomplete matrices was proposed by Harker \citep{Harker1987amoq}.
Similar extension of GMM for incomplete PC matrices can be found in
\citep{kulakowski2020otgm} or equivalently for Logarithmic Least-Squares
method in \citep{Bozoki2010ooco,Tone1993llsm}.

\subsection{Inconsistency measures}

In the literature we can find a lot of inconsistency indices for complete
matrices, however the first one and probably the most popular one
is consistency index CI introduced by Saaty \citep{Saaty1977asmf}.
For a given $n\times n$ PC matrix $C$ Consistency Index is defined
as 
\[
CI(C)=\frac{\rho(C)-n}{n-1},
\]
where $\rho(C)$ is a spectral radius of $C$, which is also its eigenvalue.

Following the idea of Saaty in \citep{Harker1987amoq} Harker defined
the Consistency Index for incomplete PCMs as 
\[
\overline{CI}(C)=\frac{\rho(H)-n}{n-1},
\]
where $H$ is an auxiliary matrix resulting from $C$: 
\begin{equation}
h_{ij}=\begin{cases}
1+s_{i}, & \textnormal{if }i=j\\
0, & \textnormal{if }c_{ij}=\ ?\\
c_{ij}, & \textnormal{otherwise}
\end{cases},\label{Hmatrix}
\end{equation}
and $s_{i}$ denotes the number of ? in the $i$-th row.

\subsection{Some notions and facts from linear algebra}

Let us recall some useful definitions and theorems from the linear
algebra.
\begin{rem}
\label{eig0} A square matrix is invertible if and only if $0$ is
not its eigenvalue. 
\end{rem}

\begin{defn}
For a given $k\times k$ incomplete matrix $A=[a_{ij}]$ we define
the corresponding directed graph $G(A)=(V,E)$, with the set of vertices
$V=\{1,\ldots,k\}$ such that an edge $i\rightarrow j$ belongs to
$E$ if and only if $a_{ij}$ is known. 
\end{defn}

\begin{defn}
We call a $k\times k$ matrix $A$ irreducible if its corresponding
graph $G(A)$ is strongly connected. 
\end{defn}

\begin{thm}
\label{lin_rho} If $\lambda\in\mathbb{C}$ is an eigenvalue of a
$k\times k$ matrix $A$ and $\alpha\in\mathbb{C}$ then $\alpha(\lambda-1)$
is an eigenvalue of $\alpha(A-I_{k})$. 
\end{thm}

\begin{defn}
For a $k\times k$ matrix $A$ we define its spectral radius as 
\[
\rho(A)=\max\{|\lambda_{1}|,\ldots,|\lambda_{k}|\},
\]
where $\lambda_{1},\ldots,\lambda_{k}$ are eigenvalues of $A$. 
\end{defn}

\begin{thm}[Perron-Frobenius]
\label{PF} If $A$ is a $k\times k$ non-negative matrix such that
for some $p\in\mathbb{N}$ the matrix $(I_{k}+A)^{p}$ is positive
then 
\begin{itemize}
\item $\rho(A)$ is a positive single eigenvalue of $A$. 
\item there exists an eigenvector corresponding to $\rho(A)$ with positive
coordinates. 
\end{itemize}
\end{thm}

For two matrices $A=[a_{ij}]$ and $B=[b_{ij}]$ we will write $A\leq B$
if for each $i,j$ the inequalty $a_{ij}\leq b_{ij}$ holds. Similarly,
the notation $A<B$ means that $A\leq B$ and there exist $i,j$ such
that $a_{ij}<b_{ij}$.
\begin{thm}
\label{mono_rho} Let $A=[a_{ij}]$ and $B=[b_{ij}]$ be two square
nonnegative matrices. If $A\leq B$, then $\rho(A)\leq\rho(B)$. Moreover,
if $B$ is irreducible and $A<B$, then $\rho(A)<\rho(B)$. 
\end{thm}

The next corollary follows immediately from Theorems \ref{lin_rho}
and \ref{PF}. 
\begin{cor}
\label{rhoPC} If $C$ is a $k\times k$ PC matrix and $\alpha>0$
then 
\[
\rho(\alpha(C-I_{k}))=\alpha(\rho(C)-1).
\]
\end{cor}

\begin{thm}
\label{inv} If $B$ is a $k\times k$ matrix and $\rho(B)<1$, then
$I_{k}-B$ is invertible, the series 
\[
\sum^{\infty}_{j=0}B^{j}
\]
is convergent and 
\[
(I_{k}-B)^{-1}=\sum^{\infty}_{j=0}B^{j}.
\]
\end{thm}

For a $k\times k$ matrix we put 
\[
r_{i}=\sum_{j\neq i}|a_{ij}|.
\]
and for $i\in\{1,\ldots,k\}$ we define a Gersgorin disc as 
\[
D_{i}=\{z\in\mathbb{C}:\ |z-a_{ii}|\leq r_{i}\}.
\]

\begin{thm}[Gershgorin]
\label{Ger} If $A=[a_{ij}]$ is a $k\times k$ matrix and $\lambda\in\mathbb{C}$
is its eigenvalue, then 
\[
\lambda\in\bigcup^{n}_{i=1}D_{i}.
\]
\end{thm}

\section{Heuristic Rating Estimation Method - the complete case}\label{sec:Heuristic-Rating-Estimation}

We assume that the set of alternatives consists of two disjoint subsets:
the alternatives with unknown preferential values $A_{U}=\{a_{1},\ldots,a_{k}\}$
where \linebreak{}
$1\leq k<n$, and the reference alternatives for which the preferences
are known $A_{K}=\{a_{k+1},\ldots,a_{n}\}$. The ranking value $w(a_{l})\in\mathbb{R}_{+}$
for each $l\in\{k+1,\ldots,n\}$ is known from the very beginning.
The aim of the HRE process is to derive the unknown priorities of
the alternatives from $A_{U}$.

Let 
\[
C_{k}=\left[\begin{array}{cccc}
1 & c_{12} & \cdots & c_{1k}\\
c_{21} & 1 & \vdots & c_{2k}\\
\vdots & \vdots & \ddots & c_{k-1,k}\\
c_{k1} & \cdots & c_{k,k-1} & 1
\end{array}\right]
\]
be a $k\times k$ submatrix of an $n\times n$ PC matrix $C$.

\subsection{The arithmetic HRE}

The arithmetic HRE procedure has been introduced in \citep{Kulakowski2014hrea}.
The ranking can be found by solving the equation 
\begin{equation}
A_{k}w=b,\label{eq:hre-arit}
\end{equation}
where 
\[
A_{k}=\left[\begin{array}{cccc}
1 & -\frac{1}{n-1}c_{21} & \cdots & -\frac{1}{n-1}c_{1k}\\
-\frac{1}{n-1}c_{21} & 1 & \vdots & -\frac{1}{n-1}c_{2k}\\
\vdots & \vdots & \ddots & -\frac{1}{n-1}c_{k-1,k}\\
-\frac{1}{n-1}c_{k1} & \cdots & -\frac{1}{n-1}c_{k,k-1} & 1
\end{array}\right]
\]
is a $k\times k$ auxiliary matrix and

\[
b=\left[\begin{array}{c}
\frac{1}{n-1}\sum^{n}_{j=k+1}c_{1j}w(a_{j})\\
\frac{1}{n-1}\sum^{n}_{j=k+1}c_{2j}w(a_{j})\\
\\\frac{1}{n-1}\sum^{n}_{j=k+1}c_{kj}w(a_{j})
\end{array}\right]
\]
is a constant term vector.

We can formulate sufficient conditions for $A_{k}$ to be invertible
by means of Saaty consistency index \emph{CI} of the matrix $C_{k}$.
\begin{thm}
\label{compl_a} If $1<k<n$ and $\textit{CI}(C_{k})<\frac{n-k}{k-1}$
or $k=1$ then matrix $A_{k}$ is invertible, so the equation (\ref{eq:hre-arit})
has the unique solution. 
\end{thm}

\begin{proof}
For $k=1$ the statement is obvious, so we consider $k>1$.

Assume that $\rho(C_{k})=\lambda\geq k$ is the principle eigenvalue
for $C_{k}$ (equal to its spectral radius). Since $C_{k}$ is a PC
matrix, it follows from Corollary \ref{rhoPC} that 
\[
\rho(C_{k}-I_{k})=\lambda-1
\]
and 
\[
\rho\left(\frac{1}{n-1}(C_{k}-I_{k})\right)=\frac{\lambda-1}{n-1}.
\]
Thus, according to Theorem \ref{inv}, 
\[
A_{k}=I_{k}-\frac{1}{n-1}(C_{k}-I_{k})
\]
is invertible if 
\[
\frac{\lambda-1}{n-1}<1,
\]
which is equivalent to 
\begin{equation}
\lambda<n.\label{sc1}
\end{equation}
As 
\[
\textit{CI}(C_{k})=\frac{\lambda-k}{k-1},
\]
the condition 
\[
\textit{CI}(C_{k})<\frac{n-k}{k-1}
\]
is equivalent to (\ref{sc1}), which completes the proof. 
\end{proof}

Theorem \ref{compl_a} implies the following corollary:
\begin{cor}
If $C_{k}$ is consistent, that is $CI(C_{k})=0$, then the equation
(\ref{eq:hre-arit}) has the unique solution. 
\end{cor}

In particular, if $C$ is consistent then for each $k<n$ matrices
$C_{k}$ are consistent, too.

Notice that hardly ever $A_{k}$ is singular, but the following example
shows that it is possible.
\begin{example}
Let 
\[
C=\left[\begin{array}{ccccc}
1 & 26+5\sqrt{27} & 1 & * & *\\
26-5\sqrt{27} & 1 & 1 & * & *\\
1 & 1 & 1 & * & *\\
* & * & * & 1 & *\\
* & * & * & * & 1
\end{array}\right]
\]
be a $5\times5$ PC matrix (stars in the last two columns can be replaced
with any positive numbers and the ones in the last two rows by their
inverses).

Then, for $k=3$, in the HRE procedure we obtain the matrix 
\[
{A}_{3}=\left[\begin{array}{ccc}
1 & -\frac{26+5\sqrt{27}}{4} & -\frac{1}{4}\\
-\frac{26-5\sqrt{27}}{4} & 1 & -\frac{1}{4}\\
-\frac{1}{4} & -\frac{1}{4} & 1
\end{array}\right],
\]
which is singular. This does not contradict Theorem \ref{compl_a},
because 
\[
\rho(C_{3})=5,
\]
which implies that 
\[
CI(C_{3})=2=\frac{n-k}{k-1}.
\]
This implies that the threshold of inconsistency in Theorem \ref{compl_a}
is optimal and cannot be improved. 
\end{example}

\subsection{The geometric HRE}

According to \citep{Kulakowski2015hreg}, to calculate the ranking
in the case of a geometric approach the equation 
\begin{equation}
{A}_{k}\widehat{w}={b}\label{eq:hre-geom}
\end{equation}
must be solved, where

\[
{A}_{k}=\left[\begin{array}{cccc}
(n-1) & -1 & \cdots & -1\\
-1 & (n-1) & \cdots & -1\\
-1 & -1 & \ddots & -1\\
-1 & -1 & \cdots & (n-1)
\end{array}\right]
\]
and 
\[
{b}=\left[\begin{array}{c}
(n-1)\widehat{w}(a_{1})-\sum^{k}_{i=2}\widehat{w}(a_{i})\\
(n-1)\widehat{w}(a_{2})-\sum^{k}_{i=1,i\neq2}\widehat{w}(a_{i})\\
\vdots\\
(n-1)\widehat{w}(a_{k})-\sum^{k-1}_{i=1}\widehat{w}(a_{i})
\end{array}\right].
\]
Since $\widehat{w}(a_{i})$ is the logarithmized value of $w(a_{i})$
i.e. $\widehat{w}(a_{i})=\log_{e}w(a_{i})$ then the final ranking
vector is obtained from $\widehat{w}$ by the exponential transformation
$w(a_{i})=e^{\widehat{w}(a_{i})}$ for $i=1,\ldots,k$. The solution
to a geometric HRE always exists and is optimal \citep{Kulakowski2015hreg}.

We will show it once more by means of the spectral radius.
\begin{thm}
For each $k<n$ the matrix ${A}_{k}$ is invertible. Therefore, problem
(\ref{eq:hre-geom}) always has the unique solution. 
\end{thm}

\begin{proof}
Let 
\[
{\bf 1}_{k}=\left[\begin{array}{ccc}
1 & \cdots & 1\\
\vdots & \ddots & \vdots\\
1 & \cdots & 1
\end{array}\right]
\]
be a $k\times k$ matrix of $1$s. Obviously, $\rho({\bf 1_{k}})=k$,
so Corollary \ref{rhoPC} implies that 
\[
\rho\left(\frac{1}{n-1}({\bf 1_{k}}-I_{k})\right)=\frac{k-1}{n-1}<1.
\]
Therefore, by Theorem \ref{inv} 
\[
{A}_{k}=(n-1)\left(I_{k}-\frac{1}{n-1}({\bf 1_{k}}-I_{k})\right)
\]
is invertible. 
\end{proof}

\section{Heuristic Rating Estimation Method - the incomplete case}\label{sec:Heuristic-Rating-Estimation-1}

In practice it may happen that not all comparisons between alternatives
are made, so a PC matrix may be incomplete. Of course, the necessary
condition for an incomplete PC matrix to induce a ranking is that
it is irreducible (see \citep{Koczkodaj2015pcs}). From now on we
assume that the considered PC matrix $C$ is irreducible. We also
assume that 
\[
c_{pq}=\frac{w(p)}{w(q)},
\]
for each $p,q\in\{k+1,\ldots,n\}$.

Let us denote the number of undefined comparisons in the $i$-th row
of a PC matrix $C$ by $s_{i}$. The weights we want to determine
are denoted as:

\begin{equation}
w=\left(\begin{array}{c}
w(a_{1})\\
\vdots\\
\vdots\\
w(a_{k})
\end{array}\right).\label{eq:unknown_terms_vector}
\end{equation}

\subsection{The arithmetic incomplete HRE}

The original method's idea is to find the unknown weights of alternatives
satisfying the system of linear equations:

\begin{equation}
w(a_{i})=\frac{1}{n-1}\left(\sum_{\substack{j=1,i\neq j\\
c_{ij}\neq?
}
}c_{ij}w(a_{j})+\sum_{\substack{j=1,i\neq j\\
c_{ij}=?
}
}w(a_{i})\right),\,\,\text{for}\,\,i=1,\ldots,n\label{ar_inc_hre}
\end{equation}

Let us define the following matrix $k\times k$ matrix $D_{k}$: 
\begin{equation}
d_{ij}=\begin{cases}
c_{ij}, & \text{if}\,\,c_{ij}\neq\ ?\\
0, & \text{if}\,\,c_{ij}=\ ?
\end{cases}.\label{Dmatrix}
\end{equation}

The arithmetic procedure (\ref{ar_inc_hre}) can be rewritten in the
form of the matrix equation: 
\begin{equation}
A_{k}w=b,\label{eq:hre-addit-eq-system}
\end{equation}
where

\begin{equation}
{A}_{k}=\left(\begin{array}{cccc}
1 & -\frac{1}{n-s_{1}-1}d_{1,2} & \cdots & -\frac{1}{n-s_{1}-1}d_{1,k}\\
-\frac{1}{n-s_{2}-1}d_{2,1} & 1 & \cdots & -\frac{1}{n-s_{2}-1}d_{2,k}\\
\vdots & \vdots & \vdots & \vdots\\
-\frac{1}{n-s_{k-1}-1}d_{k-1,1} & \cdots & \ddots & -\frac{1}{n-s_{k-1}-1}d_{k-1,k}\\
-\frac{1}{n-s_{k}-1}d_{k,1} & \cdots & -\frac{1}{n-s_{k}-1}d_{k,k-1} & 1
\end{array}\right),\label{eq:25-eq-1}
\end{equation}
and 
\begin{equation}
b=\left(\begin{array}{c}
\frac{1}{n-s_{1}-1}c_{1,k+1}w(a_{k+1})+\ldots+\frac{1}{n-s_{1}-1}c_{1,n}w(a_{n})\\
\frac{1}{n-s_{2}-1}c_{2,k+1}w(a_{k+1})+\ldots+\frac{1}{n-s_{2}-1}c_{2,n}w(a_{n})\\
\vdots\\
\frac{1}{n-s_{k}-1}c_{k,k+1}w(a_{k+1})+\ldots+\frac{1}{n-s_{k}-1}c_{k,n}w(a_{n})
\end{array}\right),\label{eq:constant_terms_vector}
\end{equation}
If $C$ is irreducible and consistent we may properly define a complete
and consistent matrix $\hat{C}$, whose elements are given as follows:
\[
\hat{c}_{ij}=\begin{cases}
c_{ij}, & \text{if}\,\,c_{ij}\neq\ ?\\
c_{i_{1}i_{2}}\cdot\ldots\cdot c_{i_{m-1}i_{m}}, & \text{if}\,\,c_{ij}=\ ?,\ i_{1}=i,i_{m}=j\\
 & \textnormal{ and for each }p\in\{1,\ldots,m-1\}\ c_{i_{p}i_{p+1}}\neq\ ?
\end{cases}.
\]

\begin{thm}
\label{incomp_cons} If $C$ is an irreducible and consistent PC matrix,
then vector (\ref{eq:unknown_terms_vector}), whose coordinates are
given by 
\[
w(a_{p})=\hat{c}_{pn}\cdot w(a_{n})
\]
for $p\in\{1,\ldots,k\}$ is the unique solution of (\ref{ar_inc_hre}). 
\end{thm}

\begin{proof}
We will consider two cases.

If $i\leq k$, then 
\begin{eqnarray*}
\frac{1}{n-1}\left(\sum_{\substack{j=1,i\neq j\\
c_{ij}\neq?
}
}c_{ij}w(a_{j})+\sum_{\substack{j=1,i\neq j\\
c_{ij}=?
}
}w(a_{i})\right)=\\
=\frac{1}{n-1}\left(\sum_{\substack{j\leq k,i\neq j\\
c_{ij}\neq?
}
}c_{ij}w(a_{j})+\sum_{\substack{j>k,i\neq j\\
c_{ij}\neq?
}
}c_{ij}w(a_{j})+\sum_{\substack{j=1,i\neq j\\
c_{ij}=?
}
}w(a_{i})\right)=\\
=\frac{1}{n-1}\left(\sum_{\substack{j\leq k,i\neq j\\
c_{ij}\neq?
}
}\hat{c}_{ij}\hat{c}_{jn}w(a_{n})+\sum_{\substack{j>k,i\neq j\\
c_{ij}\neq?
}
}\hat{c}_{ij}\frac{w(a_{j})}{w(a_{n})}w(a_{n})+\sum_{\substack{j=1,i\neq j\\
c_{ij}=?
}
}w(a_{i})\right)=\\
=\frac{1}{n-1}\left(\sum_{\substack{j=1,i\neq j\\
c_{ij}\neq?
}
}\hat{c}_{in}w(a_{n})+\sum_{\substack{j=1,i\neq j\\
c_{ij}=?
}
}w(a_{i})\right)=\\
=\frac{1}{n-1}\left(\sum_{\substack{j=1,i\neq j\\
c_{ij}\neq?
}
}w(a_{i})+\sum_{\substack{j=1,i\neq j\\
c_{ij}=?
}
}w(a_{i})\right)=\frac{1}{n-1}\sum_{\substack{j=1,i\neq j}
}w(a_{i})=w(a_{i}).
\end{eqnarray*}
Similarly, if $i>k$, then 
\begin{eqnarray*}
\frac{1}{n-1}\left(\sum_{\substack{j=1,i\neq j\\
c_{ij}\neq?
}
}c_{ij}w(a_{j})+\sum_{\substack{j=1,i\neq j\\
c_{ij}=?
}
}w(a_{i})\right)=\\
=\frac{1}{n-1}\left(\sum_{\substack{j=1,i\neq j\\
c_{ij}\neq?
}
}\hat{c}_{in}w(a_{n})+\sum_{\substack{j=1,i\neq j\\
c_{ij}=?
}
}w(a_{i})\right)=\\
=\frac{1}{n-1}\left(\sum_{\substack{j=1,i\neq j\\
c_{ij}\neq?
}
}\frac{w(a_{i})}{w(a_{n})}w(a_{n})+\sum_{\substack{j=1,i\neq j\\
c_{ij}=?
}
}w(a_{i})\right)=\\
=\frac{1}{n-1}\sum_{\substack{j=1,i\neq j}
}w(a_{i})=w(a_{i}).
\end{eqnarray*}

The uniqueness of the solution of (\ref{ar_inc_hre}) follows from
the consistency of $C$ and the fact that the value of $w(a_{n})$
is fixed. 
\end{proof}

Similarly to the complete case let us define a $k\times k$ submatrix
of $C$: 
\[
C_{k}=\left[\begin{array}{cccc}
1 & c_{12} & \cdots & c_{1k}\\
c_{21} & 1 & \vdots & c_{2k}\\
\vdots & \vdots & \ddots & c_{k-1,k}\\
c_{k1} & \cdots & c_{k,k-1} & 1
\end{array}\right].
\]

For each $j\in{1,\ldots,k}$ let $s_{j}$ denote the number of missing
entries in the $j$-th row of $C_{k}$. Let 
\[
s_{\mathrm{MAX}}=\max_{j\in\{1,\ldots,k\}}s_{j}
\]
and 
\[
s_{\mathrm{MIN}}=\min_{j\in\{1,\ldots,k\}}s_{j}.
\]

Now we are ready to formulate a result similar to Theorem \ref{compl_a}.
\begin{thm}
\label{incompl_a} If $1<k<n$ and $\overline{\mathrm{\mathit{CI}}}(C_{k})<\frac{n-k-s_{\mathrm{MAX}}+s_{\mathrm{MIN}}}{k-1}$
or $k=1$ then matrix $A_{k}$ is invertible, so the equation (\ref{eq:hre-addit-eq-system})
has the unique solution. 
\end{thm}

\begin{proof}
For $k=1$ the statement is obvious. For $k>1$ let us define $k\times k$
matrices $H_{k}$ and $D_{k}$ according to (\ref{Hmatrix}) and (\ref{Dmatrix}),
as well as $R_{k}$ given by

\begin{equation}
{R}_{k}=\left(\begin{array}{cccc}
1 & \frac{d_{12}}{n-s_{1}-1} & \cdots & \frac{d_{1,k}}{n-s_{1}-1}\\
\frac{d_{2,1}}{n-s_{2}-1} & 1 & \cdots & \frac{d_{2,k}}{n-s_{2}-1}\\
\vdots & \vdots & \vdots & \vdots\\
\frac{d_{k-1,1}}{n-s_{k-1}-1} & \cdots & \ddots & \frac{d_{k-1,k}}{n-s_{k-1}-1}\\
\frac{d_{k,1}}{n-s_{k}-1} & \cdots & \frac{d_{k,k-1}}{n-s_{k}-1} & 1
\end{array}\right).\label{Rmatrix}
\end{equation}

If 
\[
\overline{\mathrm{\mathit{CI}}}(C_{k})=\frac{\rho(H_{k})-k}{k-1}<\frac{n-k-s_{\mathrm{MAX}}+s_{\mathrm{MIN}}}{k-1},
\]
then 
\[
\rho(H_{k})<n-s_{\mathrm{MAX}}+s_{\mathrm{MIN}},
\]
which implies that 
\[
\rho(D_{k}+s_{\mathrm{MIN}}I_{k})=\rho(D_{k})+s_{\mathrm{MIN}}<n-s_{\mathrm{MAX}}+s_{\mathrm{MIN}},
\]
since 
\[
H_{k}=D_{k}+\left(\begin{array}{cccc}
s_{1} & 0 & \cdots & 0\\
0 & s_{2} & \cdots & 0\\
\vdots & \ddots & \ddots & \vdots\\
0 & \cdots & 0 & s_{k}
\end{array}\right)\geq D_{k}+s_{\mathrm{MIN}}I_{k}.
\]
As a consequence, 
\[
\rho(D_{k})<n-s_{\mathrm{MAX}},
\]
which is equivalent to 
\[
\frac{\rho(D_{k})-1}{n-s_{\mathrm{MAX}}-1}<1
\]
and to 
\begin{equation}
\frac{1}{2}\left(1+\frac{\rho(D_{k})-1}{n-s_{\mathrm{\mathrm{MAX}}}-1}\right)<1.\label{eq_aux}
\end{equation}
From (\ref{Rmatrix}) we have that 
\[
R_{k}\leq\left(\begin{array}{cccc}
1 & \frac{d_{12}}{n-s_{\mathrm{MAX}}-1} & \cdots & \frac{d_{1,k}}{n-s_{\mathrm{MAX}}-1}\\
\frac{d_{2,1}}{n-s_{\mathrm{MAX}}-1} & 1 & \cdots & \frac{d_{2,k}}{n-s_{\mathrm{MAX}}-1}\\
\vdots & \vdots & \vdots & \vdots\\
\frac{d_{k-1,1}}{n-s_{\mathrm{MAX}}-1} & \cdots & \ddots & \frac{d_{k-1,k}}{n-s_{\mathrm{MAX}}-1}\\
\frac{d_{k,1}}{n-s_{\mathrm{MAX}}-1} & \cdots & \frac{d_{k,k-1}}{n-s_{\mathrm{MAX}}-1} & 1
\end{array}\right)=I_{k}+\frac{D_{k}-I_{k}}{n-s_{\mathrm{MAX}}-1},
\]
so 
\[
\rho\left(\frac{R_{k}}{2}\right)=\frac{1}{2}\rho(R_{k})\leq\frac{1}{2}\left(1+\frac{\rho(D_{k})-1}{n-s_{\mathrm{MAX}}-1}\right).
\]
Inequality (\ref{eq_aux}) implies that 
\[
\rho\left(\frac{R_{k}}{2}\right)<1.
\]
Since 
\[
A_{k}=2I_{k}-R_{k}=2\left(I_{k}-\frac{R_{k}}{2}\right),
\]
and from Theorem \ref{inv} it follows that $A_{k}$ is invertible. 
\end{proof}

Unlike in the case of Theorem \ref{compl_a}, the right side of the
inequality in Theorem~\ref{incompl_a} does not have to be positive,
so Theorem \ref{incomp_cons} not always follows from Theorem~\ref{incompl_a}.
However we can formulate the following corollaries, whose assumptions
are sufficient to apply HRE to a consistent incomplete PC matrix.
\begin{cor}
If $s_{\textrm{MAX}}=s_{\textrm{MIN}}$ and $\overline{\textit{CI}}(C_{k})<\frac{n-k}{k-1}$
or $k=1$ then matrix $A_{k}$ is invertible. 
\end{cor}

\begin{cor}
If $k\leq\frac{n+1}{2}$ and $\overline{\textit{CI}}(C_{k})<\frac{n-2k+2}{k-1}$
or $k=1$ then matrix $A_{k}$ is invertible. 
\end{cor}

\begin{proof}
Due to irreducibility of $C_{k}$, we have 
\[
0\leq s_{\textrm{MIN}}\leq s_{\textrm{MAX}}\leq k-2,
\]
so 
\[
\frac{n-2k+2}{k-1}\leq\frac{n-k-s_{MAX}+s_{MIN}}{k-1},
\]
and the thesis follows from Theorem \ref{incompl_a}. 
\end{proof}

The next example shows that, similarly to the complete case, the incomplete
arithmetic HRE procedure in very special cases may not work. On the
other hand, it proves that the inconsistency threshold in Theorem
\ref{incompl_a} is optimal.
\begin{example}
Let 
\[
C=\left[\begin{array}{cccccccc}
1 & 1 & ? & ? & ? & 1351-780\sqrt{3} & * & *\\
1 & 1 & 1 & ? & ? & ? & * & *\\
? & 1 & 1 & 1 & ? & ? & * & *\\
? & ? & 1 & 1 & 1 & ? & * & *\\
? & ? & ? & 1 & 1 & 1 & * & *\\
1351+780\sqrt{3} & ? & ? & ? & 1 & 1 & * & *\\
* & * & * & * & * & * & 1 & *\\
* & * & * & * & * & * & * & 1
\end{array}\right]
\]
be an $8\times8$ PC matrix.

Then, for $k=6$, in the HRE procedure we obtain the matrix 
\[
A_{6}=\left[\begin{array}{crrrrc}
1 & -\frac{1}{4} & 0 & 0 & 0 & \frac{780\sqrt{3}-1351}{4}\\
-\frac{1}{4}\;\;\; & 1 & -\frac{1}{4} & 0 & 0 & 0\\
0 & -\frac{1}{4} & 1 & -\frac{1}{4} & 0 & 0\\
0 & 0 & -\frac{1}{4} & 1 & -\frac{1}{4} & 0\\
0 & 0 & 0 & -\frac{1}{4} & 1 & -\frac{1}{4}\;\;\;\\
-\frac{780\sqrt{3}+1351}{4} & 0 & 0 & 0 & -\frac{1}{4} & 1
\end{array}\right],
\]
which is singular.\par Let us consider the matrix 
\[
H_{6}=\left[\begin{array}{crrrrc}
4 & 1 & 0 & 0 & 0 & 1351-780\sqrt{3}\\
1 & 4 & 1 & 0 & 0 & 0\\
0 & 1 & 4 & 1 & 0 & 0\\
0 & 0 & 1 & 4 & 1 & 0\\
0 & 0 & 0 & 1 & 4 & 1\\
1351+780\sqrt{3} & 0 & 0 & 0 & 1 & 4
\end{array}\right]
\]
Its spectral radius equals 
\[
\rho(H_{4})=8,
\]
while 
\[
s_{\textrm{MAX}}=3=s_{\textrm{MIN}}.
\]
Therefore, 
\[
\overline{\textit{CI}}(C_{6})=\frac{8-6}{6-1}=\frac{2}{5}=\frac{n-k-s_{\textrm{MAX}}+s_{\textrm{MIN}}}{k-1},
\]
which does not violate Theorem \ref{incompl_a}, but shows that the
estimation of inconsistency cannot be improved.
\end{example}

\subsection{The geometric incomplete HRE}

The geometric procedure for incomplete PC matrices comes down to the
solution of the following matrix equation: 
\begin{equation}
{A}_{k}{w}=b,\label{eq:hre-geom-matrix-incompl}
\end{equation}
where 
\[
{A}_{k}=\left[\begin{array}{cccc}
(n-s_{1}-1) & q_{1,2} & \cdots & q_{1,k}\\
\vdots & \ddots &  & \vdots\\
\vdots &  & \ddots & \vdots\\
q_{k,1} & q_{k,2} & \cdots & (n-s_{k}-1)
\end{array}\right],
\]
with 
\[
q_{ij}=\begin{cases}
-1, & \text{if}\,\,c_{ij}\neq\ ?\\
0, & \text{if}\,\,c_{ij}=\ ?
\end{cases},
\]
and 
\[
b=\left[\begin{array}{c}
b_{1}\\
b_{2}\\
\vdots\\
b_{k}
\end{array}\right]
\]
is a vector depending on the known comparisons $c_{ij}$ and the weights
of the reference alternatives.
\begin{thm}
The matrix ${A}_{k}$ is invertible, thus problem (\ref{eq:hre-geom-matrix-incompl})
always has the unique solution. 
\end{thm}

\begin{proof}
Let us note that for all $i\in\{1,\ldots,n\}$ the radius of the $i$-th
Gersgorin disc $D_{i}$ for ${A}_{k}$ equals 
\[
r_{i}=\sum_{\{j\neq i:\ c_{ij}\neq\ ?\}}|-1|=k-1-s_{i},
\]
and its center is $n-s_{i}-1$. Since 
\[
\forall i\in\{1,\ldots,k\}\ |n-s_{i}-1-0|>r_{i},
\]
Theorem \ref{Ger} implies that $0$ is not an eigenvalue of the matrix
${A}_{k}$, which, by Remark \ref{eig0}, completes the proof. 
\end{proof}

\section{Summary}

In the paper we have provided sufficient conditions for a Heuristic
Rating Estimation method to be used. For the arithmetic HRE we have
estimated the inconsistency thresholds which guarantee that the method
will work. In the case of a complete PCM it depends only on the total
number of alternatives and the number of non-referential ones. In
the case of incomplete PCMs also the minimum and maximum numbers of
missing entries in rows matters. The estimations of inconsistency
which imply the possibility of HRE procedure application are optimal,
which has been shown by appropriate examples. As far as the geometric
HRE is concerned, both algorithms for complete and incomplete PCMs
produce rankings for each input.

\section{Acknowledgments}

The research has been supported by the National Science Centre, Poland
within the grant VIRGO 2024/55/B/HS4/00860. Jacek Szybowski was also
supported by AGH University of Krakow (task no. 11.11.420.004).

\section{References}

\bibliographystyle{plain}
\addcontentsline{toc}{section}{\refname}\bibliography{papers_biblio_reviewed}

\end{document}